\documentclass{article} 
\usepackage{iclr2027_conference,times}
\iclrfinalcopy

\usepackage{amsmath,amsfonts,bm}

\def\eqref#1{equation~\ref{#1}}

\def\1{\bm{1}}

\DeclareMathAlphabet{\mathsfit}{\encodingdefault}{\sfdefault}{m}{sl}
\SetMathAlphabet{\mathsfit}{bold}{\encodingdefault}{\sfdefault}{bx}{n}

\usepackage{hyperref}
\usepackage{url}

\usepackage[utf8]{inputenc}
\usepackage[T1]{fontenc}  
\usepackage{hyperref}      
\usepackage{url}         
\usepackage{booktabs}     
\usepackage{amsfonts}     
\usepackage{nicefrac} 
\usepackage{microtype}   
\usepackage{xcolor}       
\usepackage{siunitx}
\usepackage{graphicx}
\usepackage[most]{tcolorbox}
\usepackage{pifont}
\usepackage{enumitem}
\usepackage{algorithm}
\usepackage{algorithmic}
\usepackage{wrapfig}
\usepackage{booktabs}
\usepackage[table]{xcolor}
\usepackage{marvosym}
\usepackage{caption}
\usepackage{booktabs}
\usepackage{multirow}
\usepackage{bm}
\usepackage{graphicx}
\usepackage{booktabs}
\usepackage[table]{xcolor}
\usepackage{makecell}

\usepackage[most]{tcolorbox}

\usepackage{amsmath,amssymb}

\usepackage{amsmath,amssymb,amsthm}
\definecolor{goodgreen}{RGB}{25,135,84}
\definecolor{badred}{RGB}{200,55,55}
\definecolor{neutralgray}{RGB}{120,120,120}
\newcolumntype{C}[1]{>{\centering\let\newline\\\arraybackslash\hspace{0pt}}m{#1}}
\newcolumntype{K}[1]{>{\centering\arraybackslash}p{#1}}

\newtheorem{theorem}{Theorem}[section]

\newtheorem{definition}[theorem]{Definition}

\definecolor{badred}{RGB}{200,55,55}
\definecolor{neutralgray}{RGB}{120,120,120}

\newcommand*\xbar[2][0.75]{%
    \sbox{\myboxA}{$\m@th#2$}%
    \setbox\myboxB\null%
    \ht\myboxB=\ht\myboxA%
    \dp\myboxB=\dp\myboxA%
    \wd\myboxB=#1\wd\myboxA%
    \sbox\myboxB{$\m@th\overline{\copy\myboxB}$}%
    \setlength\mylenA{\the\wd\myboxA}%
    \addtolength\mylenA{-\the\wd\myboxB}%
    \ifdim\wd\myboxB<\wd\myboxA%
        \rlap{\hskip 0.5\mylenA\usebox\myboxB}{\usebox\myboxA}%
    \else
        \hskip -0.5\mylenA\rlap{\usebox\myboxA}{\hskip 0.5\mylenA\usebox\myboxB}%
    \fi}
\makeatother

\definecolor{darkgreen}{rgb}{0.0, 0.5, 0.0}
\usepackage{marvosym} 

\begin{document}

\title{DP-MIND: Differentially private memory-guided Long-context Inference}

\author{
Zhihao Liu$^{1}$, Yitian Cheng$^{1}$, Shuang Liu$^{2}$, Rui Wen$^{3}$,
Yuke Hu$^{4}$, Zhan Qin$^{1}$ \\[1mm]
$^{1}$Zhejiang University \\
$^{2}$Ant Group \\
$^{3}$Institute of Science Tokyo \\
$^{4}$King Abdullah University of Science and Technology}

\maketitle

\begin{abstract}

LLM-as-a-Service has become a dominant paradigm for accessing large language models without requiring local deployment. 
However, the process raises severe privacy concerns when user queries may include sensitive or confidential contexts.
Differential privacy (DP) provides a principled approach to protecting such inputs, while existing DP inference methods are primarily designed for short- or moderate-context settings and struggle with million-token-scale inputs due to limited context capacity and severe utility degradation under DP perturbation.

In this paper, we propose DP-MIND, a memory-guided framework that scales differentially private inference to million-token contexts.
DP-MIND enables private long-context inference by locally perturbing the input and processing it chunk by chunk with a compact memory that preserves task-relevant information across the sequence.
To make this memory robust to DP noise, we optimize the memory construction policy directly on perturbed inputs using reinforcement learning, guided by both final answer correctness and memory consistency.
We further recover perturbation-corrupted answer spans through a local denoising module without exposing the original context to the external model.
We provide theoretical privacy analysis of DP-MIND and further conduct extensive empirical analysis across multiple long-context benchmarks with context lengths up to the million-token scale, demonstrating that DP-MIND consistently outperforms existing DP and long-context inference baselines in both scalability and utility.
\end{abstract}


%

\section{Introduction}

Large language models (LLMs) have rapidly become a general-purpose interface for knowledge-intensive tasks, including question answering, document analysis, code generation, personal assistance, and decision support \citep{zhao2026survey,song2025injecting}. 
Rather than deploying these models locally, users increasingly access them through ``LLM-as-a-Service'' platforms \citep{brown2020language,ouyang2022training,chowdhery2022palm,qwen2offcial2025}, where prompts and contextual documents are submitted to remote model providers for inference. 
This service-oriented paradigm makes advanced LLMs broadly accessible, but it also raises a fundamental privacy concern: users often need to transmit sensitive context to remote model providers. 
Such context may contain private conversations, medical or financial records, source code, legal documents, or other confidential information \citep{chen2025survey,bai2025safeguarding}. 


Differential privacy (DP) provides a principled approach to protecting sensitive user contexts during remote LLM inference.
Existing privacy-preserving inference methods primarily focus on improving the privacy--utility trade-off.
DP-Forward perturbs intermediate representations~\citep{du2023dpforward}, while ObfusLM combines embedding obfuscation and shuffling against inversion attacks~\citep{lin2025obfuslm}.
Recent methods explore batch-level request sharing~\citep{mai2026sharedrequest} and token-level DP that limits the influence of sensitive context tokens~\citep{thareja2026dp}.
To recover utility after privacy protection, Split-and-Denoise and InferDPT employ denoising and response extraction, respectively~\citep{mai2024split,tong2025inferdpt}.
DP$^3$ further extends the investigation to multi-turn inference, addressing privacy-budget consumption and semantic consistency~\citep{zheng2026dp3}.

Although prior methods have made substantial efforts to improve the ``privacy--utility'' trade-off, they are primarily designed for standard short- or moderate-context inference and do not address the key challenges posed by million-token-scale inputs: context lengths that exceed the model's effective processing capacity, severe utility degradation caused by DP noise, and substantial inference overhead. Even when the input fits within the context window, LLMs may struggle to reliably retrieve and utilize relevant information from long contexts~\citep{liu2024lost}.

Driven by the limitations of standard transformer inference, recent non-private long context methods have explored several directions, including \textbf{context-window extension}, \textbf{sparse or linear attention} and \textbf{memory-augmented} inference. 
Context-window extension methods enlarge the usable context length through positional adaptation~\citep{chen2023position,peng2024yarn,ding2024longrope,yang2025qwen25_1m}, while sparse attention and memory-based methods improve long-context scalability through efficient attention patterns and compact memory states~\citep{beltagy2020longformer,zaheer2020bigbird,yu2026memagent}. 
These techniques demonstrate the feasibility of scaling non-private long-context inference and provide potential directions for DP long-context inference.

In this paper, we strive to bridge this gap by studying differentially private long-context inference through the lens of memory-augmented inference that are previously under-explored in DP-literature.
We propose DP-MIND, a memory-guided framework for differentially private long-context inference.
To address the challenges of limited context capacity, DP-MIND partitions the long context into locally perturbed chunks and processes them sequentially, while a compact memory propagates task-relevant information across chunks without exposing the clean context. 
The memory policy is further optimized directly under DP perturbation using reinforcement learning, guided jointly by final answer correctness and memory consistency with the corresponding clean reference memory.
In addition, DP-MIND employs a trusted local denoising module that leverages the alignment between clean and perturbed contexts to restore corrupted answer spans, improving answer utility without exposing the clean context to the external LLM.
Our main contributions are summarized as follows:
\begin{itemize}[leftmargin=*, labelindent=0.2em]

    \item 
    To the best of our knowledge, we are the first to systematically study differentially private inference in million-token-scale long-context settings, where existing DP inference methods face fundamental challenges in scalability, utility, and efficiency.

    \item 
    We propose DP-MIND, a memory-guided framework that combines chunk-wise perturbation with sequential memory propagation for DP long-context inference, and provide a comprehensive theoretical analysis of the privacy guarantee.
    \item 
    We design a DP-aware reinforcement learning objective for memory construction, with a reward that jointly captures final answer correctness and memory consistency under DP perturbation.
    We further introduce a trusted local denoising module that restores corrupted answer spans while keeping the clean context entirely local.

    \item 
    We conduct extensive experiments across multiple long-context benchmarks, privacy budgets, and context lengths up to the million-token scale. 
    The results demonstrate that DP-MIND substantially outperforms existing DP and long-context inference baselines under DP perturbation.

\end{itemize}

\section{Background and Related Work}

\subsection{Differentially Private Inference}

Existing text privacy methods primarily protect sensitive inputs by perturbing tokens or intermediate representations.
Early approaches inject calibrated noise into token embeddings and project the perturbed representations back to discrete tokens~\citep{feyisetan2020privacy,xu2020mahalanobis}.
SanText and CusText further improve private text perturbation through natural token replacement and customized output sets, respectively~\citep{yue2021differential,chen2023customized}.
For LLM inference, DP-Forward perturbs input representations during the forward pass~\citep{du2023dpforward}, while ObfusLM combines embedding obfuscation and shuffling to defend against embedding inversion attacks~\citep{lin2025obfuslm}.
Recent approaches further explore model-agnostic prompt protection through batch-level request sharing~\citep{mai2026sharedrequest} and token-level differential privacy by bounding the influence of sensitive context tokens~\citep{thareja2026dp}.
To improve inference utility, Split-and-Denoise and InferDPT introduce denoising and response extraction strategies, respectively~\citep{mai2024split,tong2025inferdpt}, while DP$^3$ addresses privacy-budget consumption and semantic consistency in multi-turn inference~\citep{zheng2026dp3}.
Despite these advances, existing approaches do not systematically address the scalability and utility challenges of DP inference over million-token contexts.

\subsection{Long-context Inference}
Existing approaches to long-context inference can be broadly categorized into three directions below.
\textbf{Context window extension} methods enlarge the usable context length of pretrained LLMs through positional interpolation or rescaling, such as Position Interpolation, YaRN, and LongRoPE~\citep{chen2023position,peng2024yarn,ding2024longrope}. 
\textbf{Sparse attention} methods reduce the quadratic cost of full attention by using local, global, or random attention patterns, as in Longformer and BigBird~\citep{beltagy2020longformer,zaheer2020bigbird}. 
\textbf{Memory-augmented} inference offers another approach by processing long contexts sequentially while maintaining compact representations of previously observed information.
MemAgent employs recurrent memory updates and reinforcement learning for long-context reasoning~\citep{yu2026memagent}, while LycheeMemory combines chunk-wise compression, selective memory retrieval, and end-to-end reinforcement learning~\citep{chen2026dynamic}.
Despite their advances in scalability, existing long-context methods are primarily designed for clean inputs and do not explicitly account for DP-induced information loss.
Under DP perturbation, simply extending the context window or reducing attention complexity cannot address corrupted evidence, while compression and memory construction may further propagate errors from noisy inputs.
These limitations highlight the need for scalable DP inference that can robustly preserve and propagate task-relevant information across long contexts.

\section{Preliminary}

\subsection{Differential Privacy}

Differential privacy (DP)~\citep{dwork2006our} provides a formal privacy guarantee, defined as follows:

\begin{definition}[$(\varepsilon,\delta)$-Differential Privacy]
A randomized mechanism $\mathcal{M}:\mathcal{D}\rightarrow\mathcal{R}$ satisfies $(\varepsilon,\delta)$-differential privacy if for any two adjacent inputs $C \sim \widehat{C}$ and any measurable set $\mathcal{S}\subseteq\mathcal{R}$,
\[
    \Pr[\mathcal{M}(C)\in \mathcal{S}]
    \leq
    e^{\varepsilon}
    \Pr[\mathcal{M}(\widehat{C})\in \mathcal{S}]
    + \delta .
\]
When $\delta=0$, the mechanism satisfies pure $\varepsilon$-differential privacy.
\end{definition}

We adopt metric differential privacy, where the privacy guarantee depends on a distance between inputs~\citep{andres2013geo,feyisetan2020privacy}.

\begin{definition}[$d_\chi$-Privacy]\label{def:dx_privacy}
Let $(\mathcal{X}, d_\chi)$ be a metric space where $d_\chi$ denotes a distance metric (e.g., Euclidean distance). A randomized mechanism $M: \mathcal{X} \to \mathcal{Y}$ satisfies $\epsilon$-$d_\chi$-privacy if for any $x, x' \in \mathcal{X}$ and any $S \subseteq \mathcal{Y}$:
\begin{equation}
    \Pr[M(x) \in S] \le e^{\epsilon \cdot d_\chi(x, x')} \Pr[M(x') \in S].
\end{equation}
\end{definition}
Given a tokenized query sequence $X=[x_1,\ldots,x_n]$, the embedding layer maps it to $\mathbf{X}=[\bm{x}_1,\ldots,\bm{x}_n] \in \mathbb{R}^{n\times d}$ via an embedding matrix $\mathbf{E}\in\mathbb{R}^{|\mathcal{V}|\times d}$, where $|\mathcal{V}|$ is the vocabulary size and $d$ is the embedding dimension. Under the $d_\chi$-privacy metric with parameter $\epsilon$ and $L_2$ distance, privacy is enforced by adding Laplacian noise $\bm{z}$ with density proportional to $\exp(-\epsilon\|z\|)$ \citep{wu2017bolt}. Specifically, the noise is sampled as $\bm{z} = lv$, where $l \sim \Gamma(d, 1/\epsilon)$ and $v$ is uniformly drawn from the unit ball $B^d$, yielding the privatized embedding $M(\bm{x}_t) = \bm{x}_t + \bm{z}$.
\subsection{Post-processing Property}

A key property of differential privacy is that privacy cannot be weakened by post-processing. 
That is, once a mechanism produces a DP output, any further randomized or deterministic computation that only depends on this DP output does not incur additional privacy loss~\citep{dwork2014algorithmic}.

\begin{theorem}[Post-processing\citep{dwork2014algorithmic}]
Let $\mathcal{M}$ be an $(\varepsilon,\delta)$-DP mechanism, and let $f$ be any randomized mapping that does not access the original private input except through $\mathcal{M}(C)$. 
Then the composed mechanism $f(\mathcal{M}(C))$ is also $(\varepsilon,\delta)$-DP.
\end{theorem}

The same principle also applies to metric DP: if $\mathcal{M}$ satisfies $\varepsilon d_{\mathcal{X}}$-privacy, then any post-processing of $\mathcal{M}(x)$ satisfies the same $\varepsilon d_{\mathcal{X}}$-privacy guarantee.

\section{Methodology}
\label{sec:methodology}

In this section, we propose \textbf{DP-MIND}, a \textbf{D}ifferentially 
\textbf{P}rivate \textbf{M}emory-guided \textbf{IN}ference framework with local 
\textbf{D}enoising for privacy-preserving long-context question answering. Given a long context $C$ and a query $Q$, DP-MIND aims to answer $Q$ without exposing the clean context to the external LLM.

\subsection{Chunk-wise DP Memory Inference}
\label{subsec:chunk_wise_dp_memory}

As shown in Figure~\ref{fig:workflow}, the first stage of DP-MIND performs chunk-wise DP memory inference. 
Given a long context $C$ and a query $Q$, DP-MIND first partitions $C$ into $N$ consecutive chunks,
\[
    C=(C_1,C_2,\ldots,C_N).
\]
Instead of exposing these clean chunks to the external LLM, DP-MIND applies a DP perturbation mechanism to the whole context and obtains the corresponding DP-perturbed chunks:
\[
    (C_1',C_2',\ldots,C_N') \leftarrow \mathcal{M}_{\mathrm{DP}}(C).
\]
Here, $C_i'$ denotes the noisy version of $C_i$ after applying DP perturbation to the original context. 
Equivalently, the mechanism perturbs the tokens in the long context and then returns the noisy chunks according to the same chunk partition. 
The external LLM only receives the DP-perturbed chunks $\{C_i'\}_{i=1}^N$ and never observes the clean chunks $\{C_i\}_{i=1}^N$.

\begin{wrapfigure}{r}{0.48\textwidth}
    \vspace{-1.3em}
    \centering
    \includegraphics[width=0.46\textwidth]{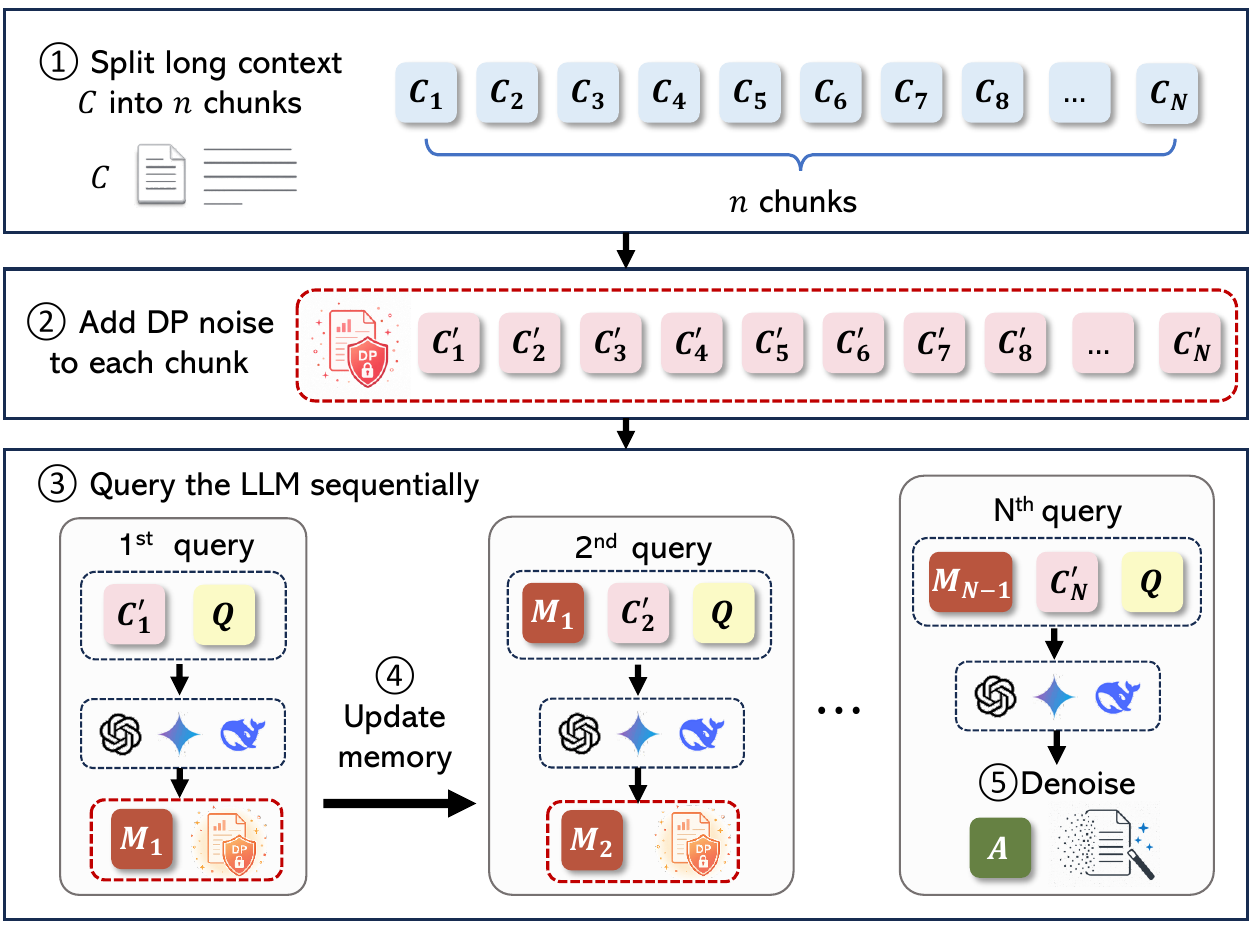}
    \caption{Workflow of DP-MIND.}
    \label{fig:workflow}
\end{wrapfigure}

After DP perturbation, DP-MIND queries the external LLM sequentially, as shown in the lower part of Figure~\ref{fig:workflow}. 
The framework maintains a memory state $M_i$ that summarizes task-relevant information from the processed noisy chunks. 
Starting from $M_0=\emptyset$, the memory is updated as
\[
    M_i \leftarrow \mathsf{LLM}(M_{i-1}, C_i', Q),
    \qquad i=1,\ldots,N-1.
\]
Thus, $M_i$ serves as a compact task-oriented state carrying information from the first $i$ DP-perturbed chunks. 
Since each memory state is computed only from DP-perturbed chunks and previous DP-derived memories, memory construction does not introduce direct access to the clean context.

At the final step, DP-MIND generates a DP answer from the last memory state, the final DP-perturbed chunk, and the query:
\[
    A_{\mathrm{DP}} \leftarrow \mathsf{LLM}(M_{N-1}, C_N', Q).
\]
Therefore, the entire external inference process only depends on the DP-perturbed context, the query, and the internal randomness of the external LLM. 
This preserves the privacy boundary of DP inference: the external LLM only accesses DP-perturbed chunks and DP-derived memories, while the clean context remains inside the trusted local environment.

\subsection{Privacy Analysis}
\label{subsec:privacy_analysis}

\begin{algorithm}[t]
\caption{$\mathcal{M}_{\mathrm{DP-MIND}}$: Chunk-wise DP Memory Inference with Local Denoising}
\label{alg:dp_mind}
\begin{algorithmic}[1]
\REQUIRE Long context $C=(w_1,\ldots,w_L)$, query $Q$, external LLM $\mathsf{LLM}$, maximum restoration span length $K$
\ENSURE Final answer $A_{\mathrm{final}}$

\STATE Apply the DP perturbation mechanism $\mathcal{M}_{\mathrm{DP}}$ in Algorithm~\ref{alg:dp_noise_injection} to obtain 
\vspace{-0.4em}
\[
    C'=(w_1',\ldots,w_L'),
    \qquad
    (C_1',\ldots,C_N')\leftarrow\mathcal{M}_{\mathrm{DP}}(C).
\]
\vspace{-1em}
\STATE Record local alignment $\Pi=\{(w_t',w_t)\}_{t=1}^{L}$.
\STATE Initialize memory $M_0\leftarrow\emptyset$.

\FOR{$i=1$ to $N-1$}
    \STATE Update memory $M_i\leftarrow \mathsf{LLM}(M_{i-1},C_i',Q)$.
\ENDFOR

\STATE Generate DP answer $A_{\mathrm{DP}}\leftarrow \mathsf{LLM}(M_{N-1},C_N',Q)$.
\STATE Denoise answer $A_{\mathrm{final}}\leftarrow \mathsf{Denoise}(A_{\mathrm{DP}},\Pi,K)$.
\RETURN $A_{\mathrm{final}}$

\end{algorithmic}
\end{algorithm}

We analyze the privacy guarantee of DP-MIND by showing that its entire external inference process is a post-processing of the underlying DP perturbation mechanism.

\begin{theorem}[Privacy Guarantee of DP-MIND]
\label{thm:privacy_mind}
Let $\mathcal{M}_{\mathrm{DP}}: C \mapsto C'$ be the local context perturbation mechanism used by DP-MIND.
Assume that $\mathcal{M}_{\mathrm{DP}}$ satisfies $\varepsilon d_{\chi}$-privacy, for a fixed query $Q$, define the DP-MIND mechanism
\[
    \mathcal{M}_{\mathrm{DP-MIND}}(C,Q)
    :=
    \mathcal{F}_{\mathrm{MIND}}
    \bigl(
        \mathcal{M}_{\mathrm{DP}}(C),Q
    \bigr),
\]
where $\mathcal{F}_{\mathrm{MIND}}$ first splits
$C'=\mathcal{M}_{\mathrm{DP}}(C)$
into $(C_1',\ldots,C_N')$, and then performs
\[
M_i \leftarrow \mathsf{LLM}(M_{i-1},C_i',Q),\ i=1,\ldots,N-1,
\qquad
A_{\mathrm{DP}} \leftarrow \mathsf{LLM}(M_{N-1},C_N',Q).
\]
The output of $\mathcal{M}_{\mathrm{DP-MIND}}$ is the complete external view of this interaction, including the DP-perturbed chunks, generated memory states, and final DP answer.
Then $\mathcal{M}_{\mathrm{DP-MIND}}$ also satisfies $\varepsilon d_{\chi}$-privacy.
\end{theorem}

\begin{proof}
Let $C'=\mathcal{M}_{\mathrm{DP}}(C)$. By assumption, $\mathcal{M}_{\mathrm{DP}}$ satisfies $\varepsilon d_{\chi}$-privacy. After obtaining $C'$, DP-MIND only performs post-processing on the perturbed context: it splits $C'$ into chunks, sequentially queries the LLM, updates the memory states, and generates the final DP answer. All these operations depend only on $C'$, the fixed query $Q$, and the internal randomness of the LLM, without accessing the clean context $C$. Therefore,
\[
    \mathcal{M}_{\mathrm{DP-MIND}}(C,Q)
    =
    \mathcal{F}_{\mathrm{MIND}}
    \bigl(
        \mathcal{M}_{\mathrm{DP}}(C),Q
    \bigr)
\]
is a post-processing of $\mathcal{M}_{\mathrm{DP}}(C)$.
By the post-processing property of differential privacy~\citep{dwork2014algorithmic}, post-processing does not increase the privacy loss.
Thus, $\mathcal{M}_{\mathrm{DP-MIND}}$ satisfies the same
$\varepsilon d_{\chi}$-privacy guarantee as
$\mathcal{M}_{\mathrm{DP}}$.
\end{proof}

\subsection{Local Answer Denoising}
\label{subsec:trusted_local_answer_denoising}

Task-relevant tokens, such as entities, numbers, and short answer spans, may be corrupted by DP perturbation, leading to utility loss in the generated answer.
To recover such errors without revealing the clean context to the external LLM, DP-MIND performs local answer denoising after inference. During local DP perturbation, DP-MIND records the token-level alignment between the clean context
$C=(w_1,\ldots,w_L)$
and the perturbed context
$C'=(w_1',\ldots,w_L')$:
\[
    \Pi=\{(w_t',w_t)\}_{t=1}^{L},
\]
which is retained only in the trusted local environment.
Using $\Pi$, we construct restoration rules for aligned spans of length up to $K$:
\[
    (w_t',\ldots,w_{t+k-1}')
    \rightarrow
    (w_t,\ldots,w_{t+k-1}),
    \qquad 1\leq k\leq K.
\]
Only perturbed spans with an unambiguous clean counterpart are included.
Given the DP answer $A_{\mathrm{DP}}$, the local module scans for matching noisy spans and replaces them with the corresponding clean spans, preferring the longest valid match.
The denoised answer is therefore obtained as
\[
    A_{\mathrm{final}}
    \leftarrow
    \mathsf{Denoise}(A_{\mathrm{DP}},\Pi,K).
\]

Algorithm~\ref{alg:span_restore} summarizes this procedure.
The denoising module only performs alignment-based span restoration rather than regenerating the answer.
Since the clean context and alignment $\Pi$ remain entirely within the trusted local environment, this step introduces no additional information leakage to the external LLM under our threat model.

\begin{algorithm}[t]
\caption{Local Answer $\mathsf{Denoise}$ in DP-MIND}
\label{alg:span_restore}
\begin{algorithmic}[1]
\REQUIRE DP answer $A_{\mathrm{DP}}$, local alignment $\Pi=\{(w_t',w_t)\}_{t=1}^{L}$, max restoration span length $K$
\ENSURE Final answer $A_{\mathrm{final}}$

\STATE Build restoration map $\mathcal{R}\leftarrow \mathsf{BuildRestoreMap}(\Pi,K)$.
\STATE Tokenize $A_{\mathrm{DP}}$ as $(a_1,\ldots,a_m)$.
\STATE Initialize $A_{\mathrm{final}}\leftarrow\emptyset$ and $r\leftarrow 1$.

\WHILE{$r\leq m$}
    \STATE $(\bar{s},k)\leftarrow
    \mathsf{LongestRestore}
    (a_r,\ldots,a_{\min(r+K-1,m)},\mathcal{R})$.
    \IF{$k>0$}
        \STATE Append $\bar{s}$ to $A_{\mathrm{final}}$ and set $r\leftarrow r+k$.
    \ELSE
        \STATE Append $a_r$ to $A_{\mathrm{final}}$ and set $r\leftarrow r+1$.
    \ENDIF
\ENDWHILE

\RETURN $A_{\mathrm{final}}$
\end{algorithmic}
\end{algorithm}

\subsection{DP-Aware Memory Policy Optimization}
\label{subsec:dp_rl_training}
We optimize the memory construction policy directly under DP perturbation using reinforcement learning.
The training objective targets two key challenges introduced by DP noise.
An \textbf{\emph{answer reward}} improves end-task utility by directly rewarding the correctness of the final denoised answer, while a \textbf{\emph{memory-consistency reward}} improves robustness by encouraging DP-derived memory states to remain close to their clean counterparts throughout sequential memory construction.
Together, the two rewards jointly optimize final prediction quality and the stability of intermediate memory under DP perturbation.
The overall reward is defined as
\[
    R_{\mathrm{final}}
    =
    R_{\mathrm{ans}}
    +
    \lambda R_{\mathrm{mem}},
\]
where $\lambda$ controls the contribution of the memory-consistency reward.

\paragraph{Answer reward $R_{\mathrm{ans}}$.}
The answer reward is designed to directly improve end-task utility under DP perturbation.
Given the DP-perturbed chunks $(C_1',\ldots,C_N')$, the policy sequentially constructs the memory trajectory
\[
    M_i^{\mathrm{DP}}
    \leftarrow
    \pi_{\theta}(M_{i-1}^{\mathrm{DP}},C_i',Q),
    \qquad i=1,\ldots,N-1,
\]
and generates the final DP answer, which is subsequently denoised in the trusted local environment:
\[
    A_{\mathrm{DP}}
    \leftarrow
    \pi_{\theta}(M_{N-1}^{\mathrm{DP}},C_N',Q),
    \qquad
    A_{\mathrm{final}}
    \leftarrow
    \mathsf{Denoise}(A_{\mathrm{DP}},\Pi),
\]
where $\Pi$ denotes the locally maintained alignment between the clean and perturbed contexts.
We then compute
\[
    R_{\mathrm{ans}}
    =
    \mathsf{Acc}(A_{\mathrm{final}},Y),
\]
where $Y$ is the gold answer.
By evaluating the final denoised answer, $R_{\mathrm{ans}}$ directly optimizes the utility of the complete DP-MIND inference pipeline.

\paragraph{Memory-consistency reward $R_{\mathrm{mem}}$.}
The memory-consistency reward is designed to improve the robustness of sequential memory construction under DP noise.
Since answer-level supervision only provides a terminal signal, it does not directly constrain how well intermediate memories preserve task-relevant information.
We therefore construct a clean reference memory trajectory:
\[
    M_i^{\mathrm{clean}}
    \leftarrow
    \pi_{\mathrm{ref}}
    (M_{i-1}^{\mathrm{clean}},C_i,Q),
    \qquad i=1,\ldots,N-1,
\]
and locally denoise the corresponding DP-derived memory:
\[
    \bar{M}_i^{\mathrm{DP}}
    \leftarrow
    \mathsf{Denoise}(M_i^{\mathrm{DP}},\Pi).
\]

We define the memory-consistency reward as the negative average embedding distance between the denoised DP-derived memory and its clean reference:
\[
    R_{\mathrm{mem}}
    =
    -
    \frac{1}{N-1}
    \sum_{i=1}^{N-1}
    \left\|
        \mathrm{Emb}(\bar{M}_i^{\mathrm{DP}})
        -
        \mathrm{Emb}(M_i^{\mathrm{clean}})
    \right\|_2,
\]
where $\mathrm{Emb}(\cdot)$ maps a memory state to its embedding representation.
A higher $R_{\mathrm{mem}}$ therefore corresponds to a smaller discrepancy from the clean reference trajectory, encouraging the policy to construct memory states that are more robust to DP perturbation.

\section{Experiments}

\subsection{Experiment Setup}

\textbf{Models and Datasets.}
We evaluate DP-MIND using the Qwen2.5 family, including 7B and 14B models~\citep{yang2024qwen25}.
For long-context scaling, we use the corresponding Qwen2.5-1M variants which supports approximately 1M tokens~\citep{yang2025qwen25_1m}.
We evaluate DP-MIND on RULER-HQA and NIAH.
RULER-HQA is constructed from HotpotQA following the RULER protocol~\citep{yang2018hotpotqa,hsieh2024ruler} and evaluates multi-hop QA over contexts from 7K to 1.8M tokens, with 128 examples per length.
NIAH follows the RULER needle-in-a-haystack setting~\citep{hsieh2024ruler} and evaluates Single, Multi-Key, and Multi-Query retrieval at 32K, 128K, and 512K tokens.


\begin{table*}[t]

\centering
\scriptsize
\caption{
Sub-EM comparison on RULER-HQA under different privacy budgets and context lengths.
Qwen2.5-7B denotes the standard baseline,
Sparse Attention denotes Longformer-style local sliding-window attention~\citep{beltagy2020longformer},
Long-Context Scaling denotes direct inference with Qwen2.5-7B-Instruct-1M~\citep{yang2025qwen25_1m}, which is a variant designed for up to one million tokens of context,
and Memory-Augmented denotes direct chunk-wise querying with recurrent memory propagation.
}
\label{tab:subem_multi_eps}

\setlength{\tabcolsep}{5.2pt}
\renewcommand{\arraystretch}{1.18}
\definecolor{bestblue}{RGB}{20,82,165}
\definecolor{bestbg}{RGB}{238,245,255}
\definecolor{missgray}{RGB}{145,145,145}

\newcommand{\ours}[1]{\textcolor{bestblue}{#1}}
\newcommand{\oursbest}[1]{\textcolor{bestblue}{\textbf{#1}}}
\newcommand{\best}[1]{\textbf{#1}}
\newcommand{\NA}{\textcolor{missgray}{--}}
\newcommand{\TODO}{\textcolor{missgray}{TODO}}

\begin{tabular}{@{}c l ccccccccc@{}}
\toprule
\textbf{Privacy Budget}
& \textbf{Method}
& \textbf{7K}
& \textbf{14K}
& \textbf{28K}
& \textbf{56K}
& \textbf{112K}
& \textbf{224K}
& \textbf{448K}
& \textbf{896K}
& \textbf{1.8M} \\
\midrule

\multirow{5}{*}{$\epsilon=6$}
& Qwen2.5-7B (baseline)
& 36.72 & 37.50 & 25.78 & \NA & \NA & \NA & \NA & \NA & \NA \\

& Sparse Attention
& 36.72 & 37.50 & 25.78 & 5.47 & 3.12 & 0.00 & \NA & \NA & \NA \\

& Long-Context Scaling
& 40.62 & 40.62 & 35.16 & 32.03 & 22.66 & 18.75 & 8.56 & 0.00 & \NA \\

& Memory-Augmented
& 38.28 & 32.03 & 33.59 & 28.91 & 25.00 & 17.97 &  8.59 & 6.25 & 0.00 \\

\rowcolor{bestbg}
& \ours{\textbf{DP-MIND}}
& \oursbest{56.25}
& \oursbest{56.25}
& \oursbest{48.44}
& \oursbest{46.09}
& \oursbest{44.53}
& \oursbest{47.66}
& \oursbest{53.12}
& \oursbest{49.22}
& \oursbest{39.06} \\

\addlinespace[2pt]
\midrule

\multirow{5}{*}{$\epsilon=7$}
& Qwen2.5-7B (baseline)
& 47.66 & 46.09 & 40.62 & \NA & \NA & \NA & \NA & \NA & \NA \\

& Sparse Attention
& 47.66 & 46.09 & 40.62 & 6.25 & 4.69 & 0.00 & \NA & \NA & \NA \\

& Long-Context Scaling
& 46.09 & 53.12 & 42.97 & 45.31 & 35.94 & 26.56 & 9.38 & 0.00 & \NA \\

& Memory-Augmented
& 58.59 & 41.41 & 44.53 & 30.47 & 30.47 & 29.69 & 27.34 & 21.09 & 17.97 \\


\rowcolor{bestbg}
& \ours{\textbf{DP-MIND}}
& \oursbest{75.78}
& \oursbest{74.22}
& \oursbest{71.09}
& \oursbest{73.44}
& \oursbest{68.56}
& \oursbest{70.31}
& \oursbest{71.88}
& \oursbest{64.81}
& \oursbest{59.38} \\

\addlinespace[2pt]
\midrule

\multirow{5}{*}{$\epsilon=8$}
& Qwen2.5-7B (baseline)
& 55.47 & 49.22 & 50.00 & \NA & \NA & \NA & \NA & \NA & \NA \\

& Sparse Attention
& 55.47 & 49.22 & 50.00 & 4.69 & 3.12 & 0.00 & \NA & \NA & \NA \\

& Long-Context Scaling
& 55.47 & 51.56 & 50.78 & 55.47 & 45.31 & 37.50 & 9.38 & 0.00 & \NA \\

& Memory-Augmented
& 59.38 & 44.53 & 42.19 & 39.06 & 42.19 & 39.84 & 29.69 & 32.81 & 22.66 \\

\rowcolor{bestbg}
& \ours{\textbf{DP-MIND}}
& \oursbest{80.47}
& \ours{75.00}
& \oursbest{77.34}
& \oursbest{80.47}
& \oursbest{73.44}
& \oursbest{77.34}
& \oursbest{72.66}
& \oursbest{68.75}
& \oursbest{63.28} \\

\bottomrule
\end{tabular}
\end{table*}

\begin{figure}[t]
    \centering
    \includegraphics[width=0.9\linewidth]{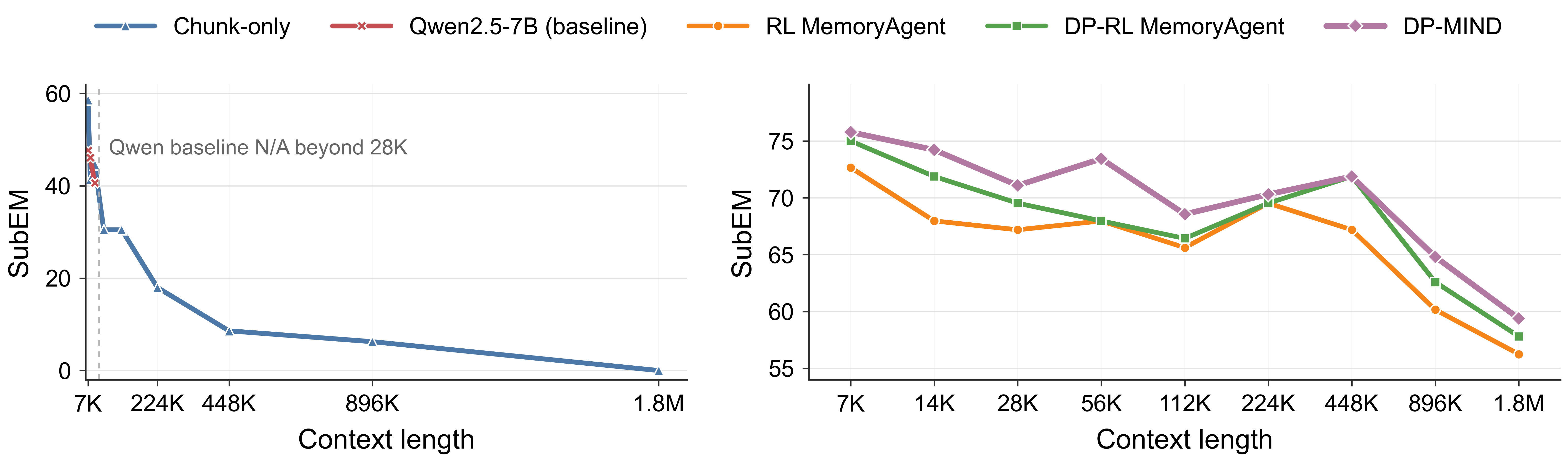}
    \caption{
Results on RULER-HQA under $\epsilon=7$. 
Chunk-only inference uses chunk-wise memory processing and can scale to long contexts, but its utility drops quickly without RL-based memory optimization. 
DP-MIND achieves the best overall performance.
}
    \label{fig:ablation_length}
\end{figure}

\subsection{Long-context Inference Task}

\textbf{DP-MIND is the only method that consistently supports ultra-long DP-perturbed contexts.}
Table~\ref{tab:subem_multi_eps} reports the SubEM results on RULER-HQA across different privacy budgets and context lengths. 
Here, ``--'' indicates that the corresponding method cannot process the input because the context length exceeds its supported context window or serving limit. 
The standard Qwen2.5-7B baseline only supports up to 28K tokens. 
Sparse Attention extends the usable length to 224K tokens, but still fails to handle 448K and longer contexts. 
Long-Context Scaling, implemented with Qwen2.5-7B-Instruct-1M, can process up to 1.8M tokens, but it relies on a model with an enlarged context window. 
In contrast, DP-MIND processes the context chunk by chunk and remains applicable across the full range from 7K to 1.8M tokens.

\textbf{DP-MIND achieves much higher utility than scalable baselines.}
Although Long-Context Scaling can process ultra-long inputs, its utility degrades rapidly as the context grows, while Sparse Attention deteriorates even earlier.
Under $\epsilon=7$, Long-Context Scaling drops from 45.31 SubEM at 56K to 26.56 at 224K and 0.00 at 896K and 1.8M, whereas DP-MIND maintains 73.44, 70.31, 64.81, and 59.38, respectively.
These results demonstrate that simply extending the context window or adopting sparse attention is insufficient under DP perturbation, while memory-guided inference preserves task-relevant information more effectively.

\textbf{DP-MIND consistently generalizes to the larger 14B backbone.}
We further evaluate DP-MIND with Qwen2.5-14B to examine whether its effectiveness generalizes to a larger model scale.
As shown in Figure~\ref{fig:ruler_hqa_14b}, DP-MIND achieves the best overall performance across all evaluated context lengths.
Compared with MemAgent, DP-MIND consistently improves performance at every context length, showing that DP-aware RL training and local answer denoising remain effective when scaling from 7B to 14B.
Moreover, the benefit of the memory-based inference framework becomes increasingly pronounced as the context grows.
Compared with direct long-context scaling, Memory-Augmented inference provides relatively modest gains in the short-context regime, but the gap widens substantially at longer contexts, reaching approximately 19.5 SubEM points at 1.8M tokens.
These results further demonstrate that the advantage of DP-MIND is not specific to the 7B backbone and that memory-guided inference becomes increasingly valuable in the ultra-long-context regime.

\begin{figure}[t]
    \centering
    \includegraphics[width=0.6\linewidth]{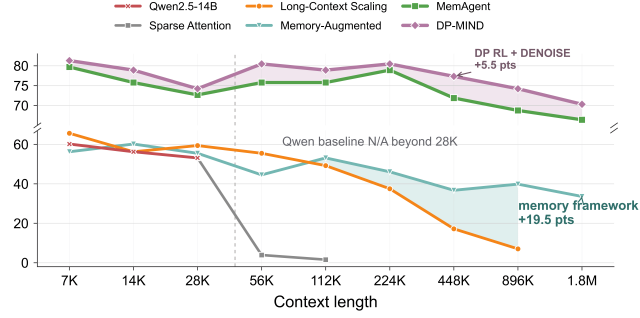}
    \caption{
    Scaling performance on RULER-HQA with Qwen2.5-14B under $\epsilon=7$.
    DP-MIND consistently achieves the best overall performance, while the advantage of memory-based inference becomes more pronounced as the context length increases.
    }
    \label{fig:ruler_hqa_14b}
\end{figure}

\subsection{Long-context Retrieval Task}

\begin{figure}
    \centering
    \includegraphics[width=0.95\linewidth]{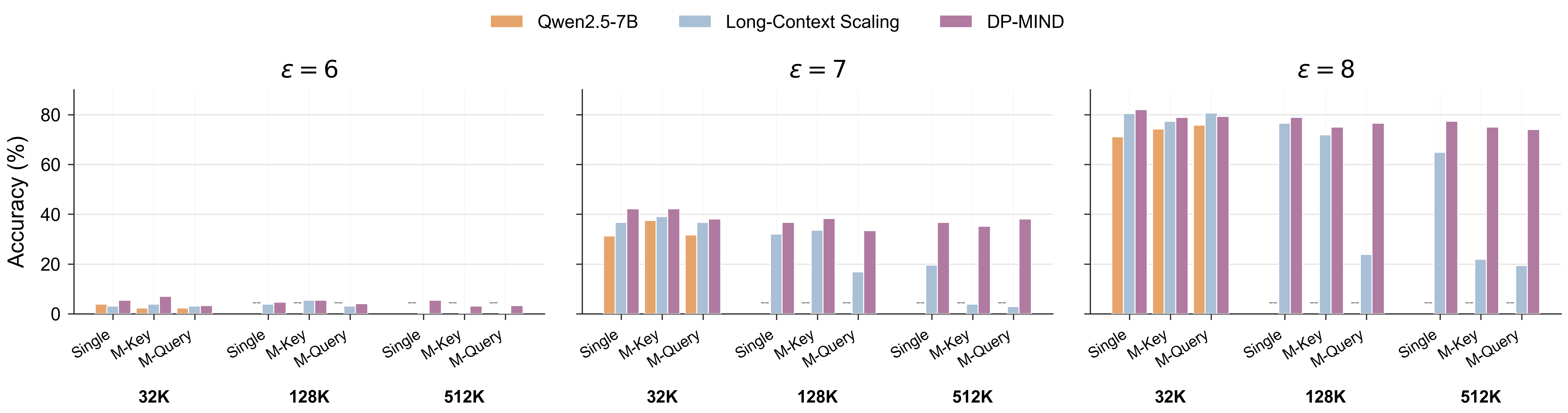}
    \caption{
NIAH results under different privacy budgets and context lengths.
Results are reported on Single, Multi-Key (M-Key), and Multi-Query (M-Query) with length ranging from 32k to 512k.
}
\label{fig:niah_results}
\end{figure}

We further extend the evaluation to NIAH to examine long-context retrieval under DP perturbation.
As shown in Figure~\ref{fig:niah_results}, DP-MIND consistently outperforms the baselines across different context lengths and privacy budgets.
\textbf{Across privacy budgets, DP-MIND remains effective under different perturbation strengths.}
When $\epsilon=6$, severe DP noise makes retrieval challenging for all methods; at $\epsilon=7$, DP-MIND maintains relatively stable performance as the context length increases, while the baselines degrade more noticeably; and at $\epsilon=8$, DP-MIND achieves the largest overall gains.
\textbf{The advantage of DP-MIND also becomes more pronounced as the retrieval context grows.}
The performance gap between DP-MIND and the baselines generally widens at longer context lengths, showing that memory-guided inference is particularly effective for preserving and retrieving sparse target information from ultra-long DP-perturbed contexts.

\subsection{Ablation Study}

\begin{figure}
    \centering
    \includegraphics[width=0.75\linewidth]{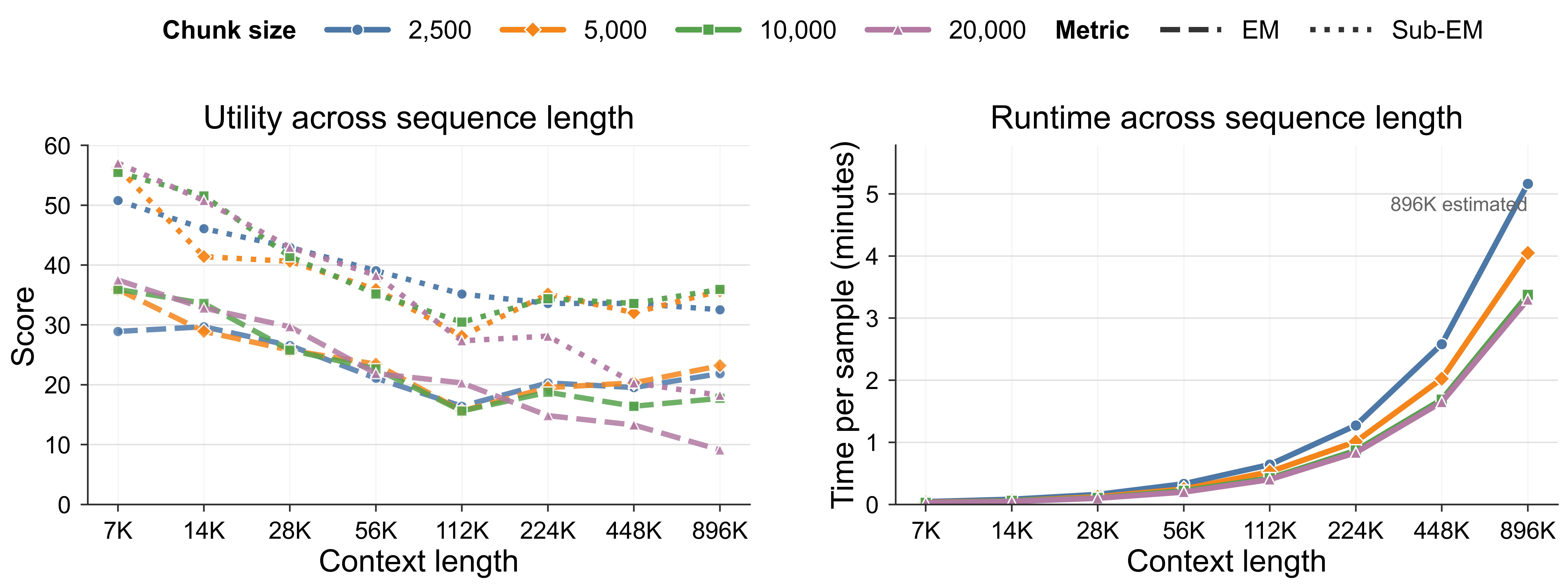}
    \caption{
Impact of chunk size on utility and runtime.
Larger chunks reduce inference time, while moderate chunk sizes generally maintain stronger EM and SubEM across long contexts.
}
    \label{fig:chunk_size}
\end{figure}

\textbf{Chunk size provides a trade-off between utility and inference efficiency.}
As shown in Figure~\ref{fig:chunk_size}, DP-MIND remains relatively stable across a wide range of chunk sizes, indicating that its performance is not overly sensitive to this hyperparameter.
For utility, chunk sizes of 5K--10K generally achieve competitive EM and SubEM across different context lengths, while an excessively large chunk size of 20K tends to degrade performance in the ultra-long regime.
In contrast, runtime is more sensitive to chunk size: larger chunks reduce the number of recurrent inference steps and therefore substantially lower the time per sample, with the efficiency gap becoming increasingly pronounced as the context length grows.
Overall, a moderate chunk size provides a favorable balance between retrieval utility and inference efficiency.

\begin{wraptable}{r}{0.44\textwidth}
\vspace{-0.3em}
\centering
\scriptsize

\definecolor{gain}{RGB}{30,130,70}

\caption{
Robustness of DP-MIND under different DP perturbation mechanisms.
}
\label{tab:perturbation_robustness}

\setlength{\tabcolsep}{2.8pt}
\renewcommand{\arraystretch}{1.12}

\newcommand{\gain}[1]{\textcolor{gain}{\tiny(+ #1)}}

\begin{tabular}{@{}l c c c@{}}
\toprule
\textbf{Perturbation}
&
\textbf{Length}
&
\textbf{Qwen-1M}
&
\textbf{DP-MIND}
\\
\midrule

\multirow{4}{*}{
\makecell[l]{
$d_{\chi}$-privacy\\
\citep{feyisetan2020privacy}
}}
& 56K  & 45.31 & 73.44 \gain{28.13} \\
& 112K & 35.94 & 68.56 \gain{32.62} \\
& 224K & 26.56 & 70.31 \gain{43.75} \\
& 448K & 9.38  & 71.88 \gain{62.50} \\
\midrule

\multirow{4}{*}{
\makecell[l]{
SanText\\
\citep{yue2021differential}
}}
& 56K  & 52.34 & 82.03 \gain{29.69} \\
& 112K & 50.78 & 78.91 \gain{28.13} \\
& 224K & 37.50 & 78.12 \gain{40.62} \\
& 448K & 6.25  & 75.00 \gain{68.75} \\
\midrule

\multirow{4}{*}{
\makecell[l]{
InferDPT\\
\citep{tong2025inferdpt}
}}
& 56K  & 57.81 & 80.47 \gain{22.66} \\
& 112K & 50.78 & 78.91 \gain{28.13} \\
& 224K & 34.38 & 78.12 \gain{43.74} \\
& 448K & 7.81  & 71.88 \gain{64.07} \\

\bottomrule
\end{tabular}

\vspace{-1em}
\end{wraptable}

\textbf{Memory propagation, DP-aware RL, and local denoising progressively improve long-context utility.}
Table~\ref{tab:ablation_length} shows that Chunk-only inference degrades rapidly as the context grows, while introducing memory propagation substantially improves long-context robustness.
Training directly on DP-perturbed inputs provides further gains, and local answer denoising enables DP-MIND to achieve the best overall performance.
We further evaluate DP-MIND on the larger Qwen2.5-14B backbone.
As shown in Figure~\ref{fig:ruler_hqa_14b}, DP-MIND consistently performs best, with the advantage of memory-based inference becoming more pronounced at longer context lengths.

\textbf{DP-MIND remains effective across different DP perturbation mechanisms.}

As shown in Table~\ref{tab:perturbation_robustness}, to evaluate its robustness to different perturbation mechanism, we further adopt the perturbation mechanisms from SanText~\citep{yue2021differential} and InferDPT~\citep{tong2025inferdpt}.
Across all three settings and all evaluated context lengths, DP-MIND consistently outperforms Qwen-1M.
The gains become particularly pronounced at 448K, reaching 62.50, 68.75, and 64.07 SubEM points under the three perturbation mechanisms, respectively.
These results demonstrate that DP-MIND remains effective across different DP perturbation mechanism.

\section{Conclusion}

In this work, we study differentially private long-context inference at the million-token scale and propose \textbf{DP-MIND}, a memory-guided framework that sequentially processes DP-perturbed chunks while propagating compact memory states across the context.
DP-MIND preserves the privacy guarantee of the underlying perturbation mechanism and improves utility through DP-aware reinforcement learning for robust memory construction and trusted local denoising.
Experiments on RULER-HQA and NIAH across different privacy budgets, perturbation mechanisms, and model scales show that DP-MIND maintains strong utility in the million-token regime, demonstrating the effectiveness of memory-guided inference for scalable private LLM inference.

\clearpage

\bibliographystyle{iclr2027_conference}
\bibliography{bib} 


\appendix

\section{Implementation Details}
\label{app:implementation}

\subsection{DP Perturbation}
\label{app:perturbation}

Before sending the context to the external LLM, DP-MIND locally applies
a token-level DP perturbation mechanism to protect the original context.
Specifically, each token is mapped to its embedding representation,
perturbed with calibrated noise, and then projected back to the nearest
token in the vocabulary.
The resulting DP-perturbed context is subsequently partitioned into
sequential chunks and used as the input to the memory-guided inference
process.
Algorithm~\ref{alg:dp_noise_injection} summarizes the complete procedure.

\begin{algorithm}[h]
\caption{Long-Context DP Perturbation Mechanism $\mathcal{M}_{\mathrm{DP}}$}
\label{alg:dp_noise_injection}
\begin{algorithmic}[1]
\REQUIRE Long context $C=(w_1,\ldots,w_L)$, number of chunks $N$,
embedding map $\phi:\mathcal{W}\rightarrow\mathbb{R}^d$,
vocabulary $\mathcal{W}$, privacy parameter $\varepsilon>0$
\ENSURE DP-perturbed chunks $(C_1',\ldots,C_N')$

\FOR{$t=1$ to $L$}
    \STATE Compute embedding $\phi_t\leftarrow\phi(w_t)$.
    \STATE Sample noise $Z_t$ with density
    $p_Z(z)\propto \exp(-\varepsilon\|z\|_2)$.
    \STATE Perturb embedding
    $\widehat{\phi}_t\leftarrow \phi_t+Z_t$.
    \STATE Project back to token
    $w_t'\leftarrow
    \operatorname*{argmin}_{u\in\mathcal{W}}
    \|\phi(u)-\widehat{\phi}_t\|_2$.
\ENDFOR

\STATE Form the perturbed context
$C'=(w_1',\ldots,w_L')$
and split it into $N$ chunks
$(C_1',\ldots,C_N')$.

\RETURN $(C_1',\ldots,C_N')$
\end{algorithmic}
\end{algorithm}

\subsection{Memory-based Inference Details}
\label{app:memory_inference}

For the chunk-based recurrent baseline, we first tokenize each input
context using the Qwen2.5-7B-Instruct tokenizer and partition it into
sequential, non-overlapping chunks of 5,000 tokens, with the last chunk
allowed to contain fewer tokens.
The chunks are processed in their original order.
For each chunk, the model receives the question, the current chunk, and
the memory accumulated from previous chunks, and updates the memory while
preserving question-relevant information.
The memory is initialized as ``No previous memory,'' and each memory
update is constrained to at most 1,024 generated tokens.
After all chunks have been processed, the final answer is generated using
the question and the final accumulated memory.
We use a temperature of 0.7 and top-$p$ of 0.95, and evaluate 128 examples
under each experimental setting.

\subsection{Prompts}
\label{app:prompts}

We use fixed prompt templates for memory updating and final answer generation across all experiments.
The memory update prompt incrementally integrates task-relevant information from each chunk, while the final answer prompt generates the prediction using only the accumulated differentially private memory.
The complete templates are provided below.

\begin{tcolorbox}[
    colback=gray!4,
    colframe=black,
    title={\textbf{Memory Update Prompt}},
    arc=2mm,
    boxrule=0.7pt,
    breakable
]

You are given a problem, a section of an article that may contain
information relevant to the problem, and a previous memory.
Read the section carefully and update the memory with any new information
that helps answer the problem.
Preserve all relevant details from the previous memory while adding useful
information from the current section.

\medskip
\texttt{<problem>}\\
\texttt{\{question\}}\\
\texttt{</problem>}

\medskip
\texttt{<memory>}\\
\texttt{\{previous\_memory\}}\\
\texttt{</memory>}

\medskip
\texttt{<section>}\\
\texttt{\{current\_chunk\}}\\
\texttt{</section>}

\medskip
\textbf{Updated memory:}

\end{tcolorbox}


\begin{tcolorbox}[
    colback=gray!4,
    colframe=black,
    title={\textbf{Final Answer Prompt}},
    arc=2mm,
    boxrule=0.7pt,
    breakable
]

You are given a problem and a differentially private memory.
The memory may contain noise or omit some details due to privacy-preserving
perturbation.
Answer the problem using only the memory, without attempting to reconstruct
any original sensitive context.

\medskip
\textbf{Output format:}\\
\texttt{"Therefore, the answer is (insert answer here)".}

\medskip
\texttt{<problem>}\\
\texttt{\{question\}}\\
\texttt{</problem>}

\medskip
\texttt{<memory>}\\
\texttt{\{final\_memory\}}\\
\texttt{</memory>}

\medskip
\textbf{Your answer:}

\end{tcolorbox}

\section{Evaluation Details}
\label{app:evaluation}

\subsection{Benchmarks}
\label{app:benchmarks}

\paragraph{RULER-HQA.}
RULER-HQA is constructed from HotpotQA~\citep{yang2018hotpotqa}
following the RULER-style synthetic construction~\citep{hsieh2024ruler}.
Each example contains supporting evidence mixed with a large number of
distractor documents, enabling controlled evaluation of long-context
extrapolation.
We synthesize 128 examples for each context scale.
The evaluated contexts range from approximately \textbf{7K} to
\textbf{1.8M} tokens.

\paragraph{NIAH.}
NIAH evaluates whether the model can retrieve a short target span from a
long haystack.
We use three variants from RULER~\citep{hsieh2024ruler}: repetitive
haystack with a numeric needle, essay haystack with a numeric needle,
and essay haystack with a UUID needle.
We evaluate 128 examples at each context length, ranging from
\textbf{32K} to \textbf{512K} tokens.

\subsection{Evaluation Metrics}
\label{app:metrics}

We normalize both predictions and gold answers by lowercasing, removing
punctuation and articles, and fixing whitespace.
EM measures exact string match after normalization.
F1 measures token-level overlap between the prediction and the gold
answer.
SubEM counts a prediction as correct if it contains the complete
normalized gold answer span.
When multiple gold answers are available, we report the maximum score
over all gold answers.
For DP-MIND, local span restoration is applied before evaluation on
answer-based tasks, and EM, F1, and SubEM are computed on the restored
answer.

\section{Additional Experiment Results}

This section provides additional experimental results on RULER-HQA to further
support the effectiveness of DP-MIND. We first present an ablation study under
$\epsilon=7$ to isolate the contribution of each component, including
DP-perturbed training and local answer denoising. We then report complete
comparisons across multiple privacy budgets and context lengths using F1, Exact
Match (EM), and Sub-Exact Match (SubEM). Overall, the results show that
DP-MIND consistently improves over sparse-attention, long-context scaling, and
memory-augmented baselines, with especially clear gains in ultra-long contexts.

\begin{table*}[h]
\definecolor{gain}{RGB}{30,130,70}
\definecolor{nogain}{RGB}{130,130,130}

\centering
\vspace{1em}
\caption{
Ablation study of DP-MIND on RULER-HQA under $\varepsilon=7$, measured by SubEM. 
Chunk-only splits the context into chunks and queries them without memory propagation. 
RL MemoryAgent denotes the original memory-based inference method~\citep{yu2026memagent}. 
DP-RL MemoryAgent is trained on DP-perturbed data. 
DP-MIND is our full method with local answer denoising. 
Numbers in parentheses denote improvements over RL MemoryAgent.
}
\label{tab:ablation_length}

\setlength{\tabcolsep}{4.6pt}
\renewcommand{\arraystretch}{1.18}

\newcommand{\up}[1]{\textcolor{gain}{\scriptsize(+#1)}}
\newcommand{\eqgain}{\textcolor{nogain}{\scriptsize(+0.00)}}

\begin{tabular}{@{}lcccc@{}}
\toprule
Length
& \makecell[c]{Chunk-only}
& \makecell[c]{RL MemoryAgent}
& \makecell[c]{DP-RL MemoryAgent}
& \makecell[c]{DP-MIND} \\
\midrule

7K   
& 58.59 
& 72.66 
& 75.00 \up{2.34}
& \textbf{75.78} \up{3.12} \\

14K  
& 41.41 
& 67.97 
& 71.88 \up{3.91}
& \textbf{74.22} \up{6.25} \\

28K  
& 44.53 
& 67.19 
& 69.53 \up{2.34}
& \textbf{71.09} \up{3.90} \\

56K  
& 30.47 
& 67.97 
& 67.97 \eqgain
& \textbf{73.44} \up{5.47} \\

112K 
& 30.47 
& 65.62 
& 66.43 \up{0.81}
& \textbf{68.56} \up{2.13} \\

224K 
& 17.97 
& 69.53 
& 69.53 \eqgain
& \textbf{70.31} \up{0.78} \\

448K 
& 8.59 
& 67.19 
& 71.88 \up{4.69}
& \textbf{71.88} \up{4.69} \\

896K 
& 6.25 
& 60.16 
& 62.59 \up{2.43}
& \textbf{64.81} \up{2.22} \\

1.8M 
& 0.00 
& 56.25 
& 57.81 \up{1.56}
& \textbf{59.38} \up{3.13} \\

\bottomrule
\vspace{1em
}
\end{tabular}
\end{table*}

\begin{table*}[h]
\centering
\scriptsize
\caption{
F1 comparison on RULER-HQA under different privacy budgets and
context lengths. All values are percentages.
}
\label{tab:f1_multi_eps}

\setlength{\tabcolsep}{5.2pt}
\renewcommand{\arraystretch}{1.18}
\definecolor{bestblue}{RGB}{20,82,165}
\definecolor{bestbg}{RGB}{238,245,255}
\definecolor{missgray}{RGB}{145,145,145}

\newcommand{\ours}[1]{\textcolor{bestblue}{#1}}
\newcommand{\oursbest}[1]{\textcolor{bestblue}{\textbf{#1}}}
\newcommand{\best}[1]{\textbf{#1}}
\newcommand{\NA}{\textcolor{missgray}{--}}
\newcommand{\TODO}{\textcolor{missgray}{TODO}}

\begin{tabular}{@{}c l ccccccccc@{}}
\toprule
\textbf{Privacy Budget}
& \textbf{Method}
& \textbf{7K}
& \textbf{14K}
& \textbf{28K}
& \textbf{56K}
& \textbf{112K}
& \textbf{224K}
& \textbf{448K}
& \textbf{896K}
& \textbf{1.8M} \\
\midrule

\multirow{5}{*}{$\epsilon=6$}
& Qwen2.5-7B (baseline)
& 36.12 & 37.84 & 27.98 & \NA & \NA & \NA & \NA & \NA & \NA \\

& Sparse Attention
& 36.12 & 37.84 & 27.98 & 4.81 & 2.56 & 0.00
& \NA & \NA & \NA \\

& Long-Context Scaling
& 43.94 & 42.29 & 36.80 & 34.10 & 27.23 & 21.69 & 12.42
& 0.00 & 0.00 \\

& Memory-Augmented
& 39.98 & 33.56 & 34.38 & 27.88 & 27.52 & 22.88 & 26.51
& 20.09 & 20.94 \\

\rowcolor{bestbg}
& \ours{\textbf{DP-MIND}}
& \oursbest{53.30}
& \oursbest{54.93}
& \oursbest{53.57}
& \oursbest{48.11}
& \oursbest{48.36}
& \oursbest{48.23}
& \oursbest{52.43}
& \oursbest{47.62}
& \oursbest{42.41} \\

\addlinespace[2pt]
\midrule

\multirow{5}{*}{$\epsilon=7$}
& Qwen2.5-7B (baseline)
& 41.70 & 45.15 & 40.90 & \NA & \NA & \NA & \NA & \NA & \NA \\

& Sparse Attention
& 41.70 & 45.15 & 40.90 & 5.27 & 3.18 & \NA
& \NA & \NA & \NA \\

& Long-Context Scaling
& 50.38 & 51.08 & 42.78 & 45.42 & 35.97 & 28.17 & 8.88
& 4.69 & 0.00 \\

& Memory-Augmented
& 51.76 & 43.30 & 40.55 & 39.46 & 29.64 & 31.73 & 31.77
& 30.74 & 21.69 \\

\rowcolor{bestbg}
& \ours{\textbf{DP-MIND}}
& \oursbest{71.49}
& \oursbest{71.92}
& \oursbest{69.46}
& \oursbest{66.16}
& \oursbest{63.62}
& \oursbest{71.28}
& \oursbest{68.68}
& \oursbest{58.11}
& \oursbest{58.11} \\

\addlinespace[2pt]
\midrule

\multirow{5}{*}{$\epsilon=8$}
& Qwen2.5-7B (baseline)
& 43.19 & 40.90 & 44.53 & \NA & \NA & \NA & \NA & \NA & \NA \\

& Sparse Attention
& 43.19 & 40.90 & 44.53 & 4.25 & 3.01 & 0.48
& \NA & \NA & \NA \\

& Long-Context Scaling
& 55.77 & 53.23 & 51.14 & 54.53 & 44.60 & 37.32 & 10.95
& 8.88 & 4.81 \\

& Memory-Augmented
& 53.17 & 41.29 & 45.09 & 38.85 & 37.47 & 40.06 & 30.74
& 34.14 & 27.73 \\

\rowcolor{bestbg}
& \ours{\textbf{DP-MIND}}
& \oursbest{77.26}
& \oursbest{72.64}
& \oursbest{74.21}
& \oursbest{74.45}
& \oursbest{69.78}
& \oursbest{72.48}
& \oursbest{69.86}
& \oursbest{65.92}
& \oursbest{62.35} \\

\bottomrule
\end{tabular}
\end{table*}

\begin{table*}[h]
\centering
\scriptsize
\caption{
Exact Match (EM) comparison on RULER-HQA under different privacy
budgets and context lengths. All values are percentages.
}
\label{tab:em_multi_eps}

\setlength{\tabcolsep}{5.2pt}
\renewcommand{\arraystretch}{1.18}
\definecolor{bestblue}{RGB}{20,82,165}
\definecolor{bestbg}{RGB}{238,245,255}
\definecolor{missgray}{RGB}{145,145,145}

\newcommand{\ours}[1]{\textcolor{bestblue}{#1}}
\newcommand{\oursbest}[1]{\textcolor{bestblue}{\textbf{#1}}}
\newcommand{\best}[1]{\textbf{#1}}
\newcommand{\NA}{\textcolor{missgray}{--}}
\newcommand{\TODO}{\textcolor{missgray}{TODO}}

\begin{tabular}{@{}c l ccccccccc@{}}
\toprule
\textbf{Privacy Budget}
& \textbf{Method}
& \textbf{7K}
& \textbf{14K}
& \textbf{28K}
& \textbf{56K}
& \textbf{112K}
& \textbf{224K}
& \textbf{448K}
& \textbf{896K}
& \textbf{1.8M} \\
\midrule

\multirow{5}{*}{$\epsilon=6$}
& Qwen2.5-7B (baseline)
& 21.09 & 22.66 & 16.41 & \NA & \NA & \NA & \NA & \NA & \NA \\

& Sparse Attention
& 21.09 & 22.66 & 16.41 & 0.78 & 0.00 & 0.00
& \NA & \NA & \NA \\

& Long-Context Scaling
& 26.56 & 25.00 & 21.88 & 17.97 & 14.84 & 14.06 & 7.03
& 2.56 & 0.00 \\

& Memory-Augmented
& 24.22 & 19.53 & 19.53 & 10.16 & 14.06 & 10.94 & 12.50
& 8.59 & 10.16 \\

\rowcolor{bestbg}
& \ours{\textbf{DP-MIND}}
& \oursbest{33.59}
& \oursbest{39.84}
& \oursbest{35.94}
& \oursbest{32.03}
& \oursbest{28.91}
& \oursbest{31.25}
& \oursbest{36.72}
& \oursbest{32.81}
& \oursbest{28.12} \\

\addlinespace[2pt]
\midrule

\multirow{5}{*}{$\epsilon=7$}
& Qwen2.5-7B (baseline)
& 24.22 & 29.69 & 29.69 & \NA & \NA & \NA & \NA & \NA & \NA \\

& Sparse Attention
& 24.22 & 29.69 & 29.69 & 0.78 & 0.00 & \NA
& \NA & \NA & \NA \\

& Long-Context Scaling
& 35.16 & 32.03 & 28.91 & 31.25 & 21.88 & 16.41 & 4.69
& 4.81 & 0.00 \\

& Memory-Augmented
& 35.94 & 28.91 & 25.78 & 23.44 & 15.62 & 19.53 & 20.31
& 12.50 & 14.84 \\

\rowcolor{bestbg}
& \ours{\textbf{DP-MIND}}
& \oursbest{55.47}
& \oursbest{55.47}
& \oursbest{54.69}
& \oursbest{47.66}
& \oursbest{48.44}
& \oursbest{57.81}
& \oursbest{50.00}
& \oursbest{44.53}
& \oursbest{45.31} \\

\addlinespace[2pt]
\midrule

\multirow{5}{*}{$\epsilon=8$}
& Qwen2.5-7B (baseline)
& 28.12 & 26.56 & 32.81 & \NA & \NA & \NA & \NA & \NA & \NA \\

& Sparse Attention
& 28.12 & 26.56 & 32.81 & 0.00 & 0.00 & 0.00
& \NA & \NA & \NA \\

& Long-Context Scaling
& 37.50 & 35.16 & 35.16 & 39.84 & 29.69 & 24.22 & 15.47
& 16.41 & 14.06 \\

& Memory-Augmented
& 35.94 & 25.00 & 28.91 & 23.44 & 20.31 & 25.78 & 19.53
& 18.75 & 17.19 \\

\rowcolor{bestbg}
& \ours{\textbf{DP-MIND}}
& \oursbest{60.94}
& \oursbest{55.47}
& \oursbest{60.16}
& \oursbest{58.59}
& \oursbest{56.25}
& \oursbest{60.16}
& \oursbest{53.12}
& \oursbest{51.56}
& \oursbest{50.00} \\

\bottomrule
\end{tabular}
\end{table*}

\begin{table*}[h]
\centering
\scriptsize
\caption{
F1 comparison on RULER-HQA for Qwen2.5-14B series under $\epsilon=7$.
}
\label{tab:qwen25_14b_f1_eps7}

\setlength{\tabcolsep}{5.2pt}
\renewcommand{\arraystretch}{1.18}
\definecolor{bestblue}{RGB}{20,82,165}
\definecolor{bestbg}{RGB}{238,245,255}
\definecolor{missgray}{RGB}{145,145,145}

\newcommand{\ours}[1]{\textcolor{bestblue}{#1}}
\newcommand{\oursbest}[1]{\textcolor{bestblue}{\textbf{#1}}}
\newcommand{\NA}{\textcolor{missgray}{--}}

\begin{tabular}{@{}l ccccccccc@{}}
\toprule
\textbf{Method}
& \textbf{7K}
& \textbf{14K}
& \textbf{28K}
& \textbf{56K}
& \textbf{112K}
& \textbf{224K}
& \textbf{448K}
& \textbf{896K}
& \textbf{1.8M} \\
\midrule

Qwen2.5-14B (baseline)
& 23.31 & 22.77 & 19.64 & \NA & \NA & \NA & \NA & \NA & \NA \\

Sparse Attention
& 23.31$^{\dagger}$ & 22.77$^{\dagger}$ & 19.64$^{\dagger}$
& 3.57 & 3.64 & 0.06 & \NA & \NA & \NA \\

Long-Context Scaling
& 48.52 & 37.13 & 34.96 & 30.40 & 27.23 & 21.21 & 15.87 & 10.25 & \NA \\

Memory-Augmented
& 53.71 & 55.05 & 53.40 & 40.97 & 44.64 & 44.27 & 36.40 & 36.35 & 32.80 \\

\rowcolor{bestbg}
\ours{\textbf{DP-MIND}}
& \oursbest{75.77}
& \oursbest{72.97}
& \oursbest{67.56}
& \oursbest{71.11}
& \oursbest{71.43}
& \oursbest{72.69}
& \oursbest{66.00}
& \oursbest{65.13}
& \oursbest{62.23} \\

\bottomrule
\end{tabular}
\end{table*}

\begin{table*}[h]
\centering
\scriptsize
\caption{
Exact Match (EM) comparison on RULER-HQA for Qwen2.5-14B series under $\epsilon=7$.
}
\label{tab:qwen25_14b_em_eps7}

\setlength{\tabcolsep}{5.2pt}
\renewcommand{\arraystretch}{1.18}
\definecolor{bestblue}{RGB}{20,82,165}
\definecolor{bestbg}{RGB}{238,245,255}
\definecolor{missgray}{RGB}{145,145,145}

\newcommand{\ours}[1]{\textcolor{bestblue}{#1}}
\newcommand{\oursbest}[1]{\textcolor{bestblue}{\textbf{#1}}}
\newcommand{\NA}{\textcolor{missgray}{--}}

\begin{tabular}{@{}l ccccccccc@{}}
\toprule
\textbf{Method}
& \textbf{7K}
& \textbf{14K}
& \textbf{28K}
& \textbf{56K}
& \textbf{112K}
& \textbf{224K}
& \textbf{448K}
& \textbf{896K}
& \textbf{1.8M} \\
\midrule

Qwen2.5-14B (baseline)
& 11.72 & 9.38 & 8.59 & \NA & \NA & \NA & \NA & \NA & \NA \\

Sparse Attention
& 11.72 & 9.38 & 8.59 & 3.12 & 3.91 & 0.00 & \NA & \NA & \NA \\

Long-Context Scaling
& 32.81 & 21.09 & 24.22 & 18.75 & 17.19 & 12.50 & 6.25 & 2.34 & \NA \\

Memory-Augmented
& 37.50 & 36.72 & 34.38 & 25.78 & 28.91 & 31.25 & 22.66 & 23.44 & 19.53 \\

\rowcolor{bestbg}
\ours{\textbf{DP-MIND}}
& \oursbest{58.59}
& \oursbest{53.12}
& \oursbest{48.44}
& \oursbest{50.78}
& \oursbest{51.56}
& \oursbest{54.69}
& \oursbest{49.22}
& \oursbest{48.44}
& \oursbest{46.40} \\

\bottomrule
\end{tabular}
\end{table*}

\begin{table*}[h]
\centering
\scriptsize
\caption{
Sub-Exact Match (Sub-EM) comparison on RULER-HQA for Qwen2.5-14B series under $\epsilon=7$.
}
\label{tab:qwen25_14b_subem_eps7}

\setlength{\tabcolsep}{5.2pt}
\renewcommand{\arraystretch}{1.18}
\definecolor{bestblue}{RGB}{20,82,165}
\definecolor{bestbg}{RGB}{238,245,255}
\definecolor{missgray}{RGB}{145,145,145}

\newcommand{\ours}[1]{\textcolor{bestblue}{#1}}
\newcommand{\oursbest}[1]{\textcolor{bestblue}{\textbf{#1}}}
\newcommand{\NA}{\textcolor{missgray}{--}}

\begin{tabular}{@{}l ccccccccc@{}}
\toprule
\textbf{Method}
& \textbf{7K}
& \textbf{14K}
& \textbf{28K}
& \textbf{56K}
& \textbf{112K}
& \textbf{224K}
& \textbf{448K}
& \textbf{896K}
& \textbf{1.8M} \\
\midrule

Qwen2.5-14B (baseline)
& 60.16 & 56.25 & 53.12 & \NA & \NA & \NA & \NA & \NA & \NA \\

Sparse Attention
& 60.16 & 56.25 & 53.12 & 3.91 & 1.56 & 0.00 & \NA & \NA & \NA \\

Long-Context Scaling
& 65.62 & 56.25 & 59.38 & 55.47 & 49.22 & 37.50 & 17.19 & 7.03 & \NA \\

Memory-Augmented
& 56.25 & 60.16 & 55.47 & 44.53 & 53.12 & 46.09 & 36.72 & 39.84 & 33.59 \\

\rowcolor{bestbg}
\ours{\textbf{DP-MIND}}
& \oursbest{81.31}
& \oursbest{78.91}
& \oursbest{74.22}
& \oursbest{80.47}
& \oursbest{78.91}
& \oursbest{80.47}
& \oursbest{77.34}
& \oursbest{74.22}
& \oursbest{70.31} \\

\bottomrule
\end{tabular}
\end{table*}

\begin{table*}[h]
\definecolor{oursbg}{RGB}{238,245,255}
\definecolor{missgray}{RGB}{145,145,145}

\centering
\scriptsize
\caption{
NIAH results under different privacy budgets and context lengths.
Results are reported on Single, Multi-Key (M-Key), and Multi-Query (M-Query) settings, where $\uparrow$ indicates higher is better.
}
\label{tab:niah_results}

\setlength{\tabcolsep}{2.4pt}
\renewcommand{\arraystretch}{1.15}

\newcommand{\niahNA}{\textcolor{missgray}{--}}

\begin{tabular}{@{}cl*{9}{c}@{}}
\toprule
\multirow{2}{*}{\textbf{Privacy Budget}}
& \multirow{2}{*}{\textbf{Method}}
& \multicolumn{3}{c}{\textbf{32K}}
& \multicolumn{3}{c}{\textbf{128K}}
& \multicolumn{3}{c}{\textbf{512K}} \\
\cmidrule(lr){3-5}
\cmidrule(lr){6-8}
\cmidrule(lr){9-11}

&
& \textbf{Single$\uparrow$}
& \textbf{M-Key$\uparrow$}
& \textbf{M-Query$\uparrow$}
& \textbf{Single$\uparrow$}
& \textbf{M-Key$\uparrow$}
& \textbf{M-Query$\uparrow$}
& \textbf{Single$\uparrow$}
& \textbf{M-Key$\uparrow$}
& \textbf{M-Query$\uparrow$} \\
\midrule

\multirow{3}{*}{$\epsilon=6$}
& Qwen2.5-7B
& 3.91 & 2.34 & 2.34
& \niahNA & \niahNA & \niahNA
& \niahNA & \niahNA & \niahNA \\

& Qwen2.5-7B-1M
& 3.12 & 3.91 & 3.12
& 3.91 & 5.47 & 3.12
& 0.00 & 0.00 & 0.39 \\

& \textbf{DP-MIND}
& \textbf{5.47} & \textbf{7.03} & \textbf{3.32}
& 4.69 & 5.47 & 4.10
& 5.47 & 3.12 & 3.32 \\
\midrule

\multirow{3}{*}{$\epsilon=7$}
& Qwen2.5-7B
& 31.25 & 37.50 & 31.64
& \niahNA & \niahNA & \niahNA
& \niahNA & \niahNA & \niahNA \\

& Qwen2.5-7B-1M
& 36.72 & 39.06 & 36.72
& 32.03 & 33.59 & 16.80
& 19.53 & 3.91 & 2.93 \\

& \textbf{DP-MIND}
& \textbf{42.19} & \textbf{42.19} & \textbf{38.09}
& 36.72 & 38.28 & 33.40
& 36.72 & 35.16 & 38.09 \\
\midrule

\multirow{3}{*}{$\epsilon=8$}
& Qwen2.5-7B
& 71.09 & 74.22 & 75.78
& \niahNA & \niahNA & \niahNA
& \niahNA & \niahNA & \niahNA \\

& Qwen2.5-7B-1M
& 80.47 & 77.34 & \textbf{80.66}
& 76.56 & 71.88 & 23.83
& 64.84 & 21.88 & 19.34 \\

& \textbf{DP-MIND}
& \textbf{82.03} & \textbf{78.91} & 79.30
& 78.91 & 75.00 & 76.56
& 77.34 & 75.00 & 74.02 \\

\bottomrule
\end{tabular}
\end{table*}

\end{document}